\documentclass[11pt]{article}

\usepackage[final]{acl}

\usepackage{times}
\usepackage{latexsym}
\usepackage{amsmath,amssymb}
\usepackage{amsthm}
\newtheorem{proposition}{Proposition}[section]

\AtBeginDocument{%
  \thickmuskip=5mu\relax
  \medmuskip=4mu\relax}

\usepackage[T1]{fontenc}

\usepackage[utf8]{inputenc}

\usepackage{microtype}

\usepackage{inconsolata}

\usepackage{graphicx}
\usepackage{booktabs,multirow}
\usepackage{placeins}
\usepackage{enumitem}

\title{DPPM: Dual-Path Parametric Memory for Personalized Language Models}

\author{
  \textbf{Yuhao Chen\textsuperscript{1}},
  \textbf{Shuochen Liu\textsuperscript{1}},
  \textbf{Jiayao Shi\textsuperscript{1}},
  \textbf{Jian Hong\textsuperscript{1,2}} \\
  \textbf{Chen Cheng\textsuperscript{2}}, 
  \textbf{Xinyun Ding\textsuperscript{2}},
  \textbf{Tao Wang\textsuperscript{2}},
  \textbf{Ya Li\textsuperscript{2}},
  \textbf{Quan Liu\textsuperscript{2}},
  \textbf{Tong Xu\textsuperscript{1}} \\
  \textsuperscript{1}University of Science and Technology of China\quad
  \textsuperscript{2}iFLYTEK Research Group \\
  \texttt{isyuhaochen@mail.ustc.edu.cn}
}

\begin{document}
\maketitle
\begin{abstract}
Long-term personalization requires language models to use interaction history to track users' preferences across sessions. Parametric memory encodes this interaction history into model parameters or adapters, reducing the need to include it in the inference context. However, independent context compilation leaves cross-session integration unspecified, while recurrent updates can attenuate earlier evidence. To address these challenges, we propose \textbf{Dual-Path Parametric Memory (DPPM)}. Its Evidence path directly pools representations of the interaction history to preserve earlier evidence, while its Delta path sequentially updates an associative state to capture changes. Fusing both outputs produces history-conditioned LoRA adapters that combine evidence accumulation with ordered revision. Across multiple backbones, DPPM outperforms the evaluated baselines, achieving \textbf{54.22\%} on PersonaMem-v2 and \textbf{86.79\%} on PrefEval. These results suggest that DPPM provides a simple and effective design choice for cross-session personalized parametric memory.\footnote{Code: \url{https://github.com/isyuhaochen/DPPM}.}
\end{abstract}

\section{Introduction}

As large language models (LLMs) become persistent personal assistants, personalization increasingly depends on long-term memory. Such memory must preserve users' preferences, backgrounds, and constraints across sessions to inform subsequent responses~\citep{liu2026permabenchmarkingpersonalizedmemory,chen2026vehiclemembenchexecutablebenchmarkmultiuser}. To support this continuity, existing systems maintain textual records and construct response contexts through summarization, retrieval, or explicit memory editing~\citep{zhong2023memorybankenhancinglargelanguage,packer2024memgptllmsoperatingsystems,chhikara2025mem0buildingproductionreadyai}. While these records are easy to inspect and maintain, their effectiveness depends on accurate extraction, retrieval, and contextual reasoning, and including more of the interaction history increases inference-time context costs.

\begin{figure}[t]
  \centering
  \includegraphics[width=\linewidth]{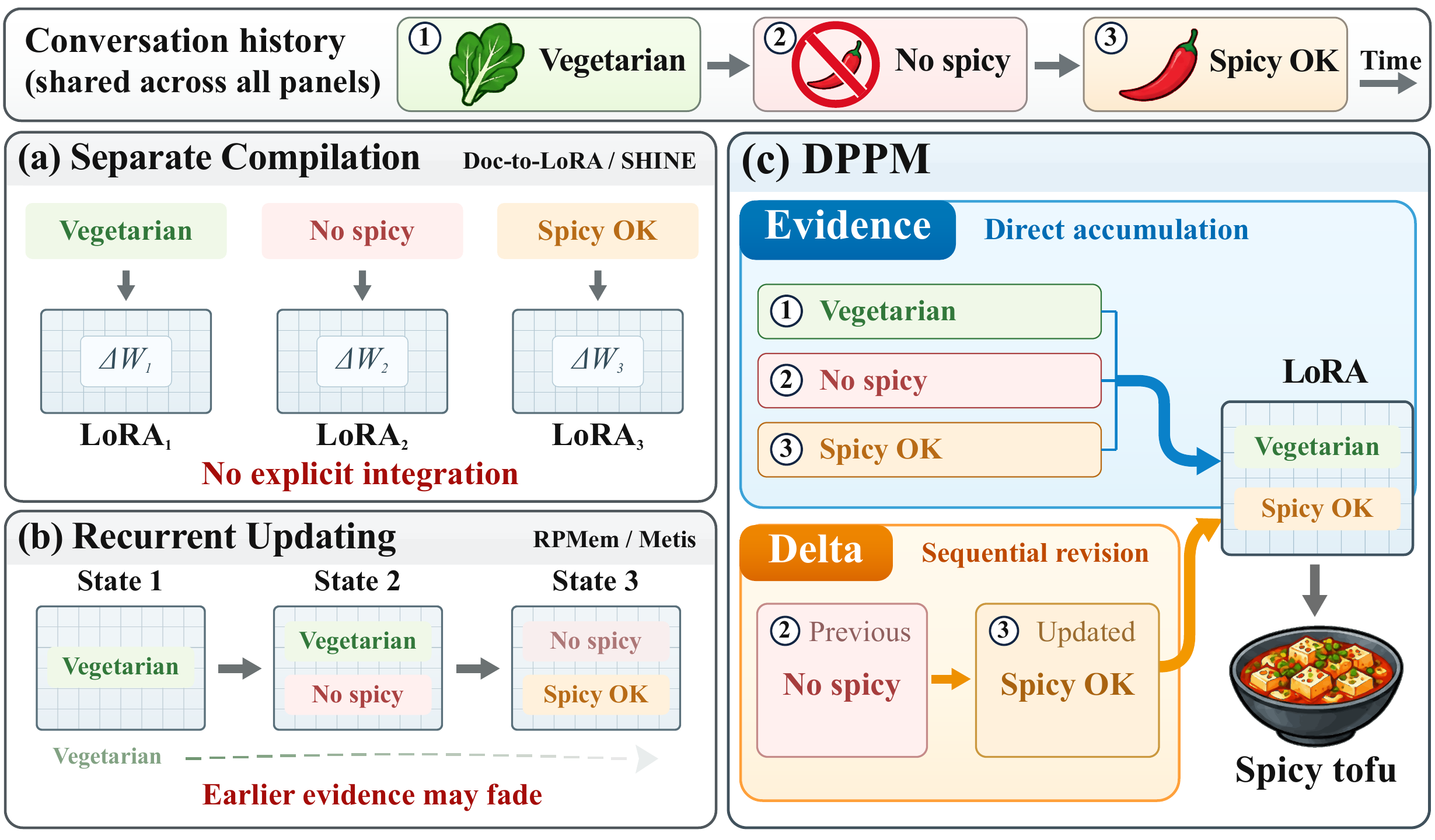}
  \caption{Illustrative comparison of cross-session memory. (a) Context compilation produces separate LoRA adapters. (b) Recurrent updates may dilute earlier evidence. (c) DPPM combines evidence accumulation and sequential revision to preserve dietary constraints while tracking preference changes.}
  \label{fig:dppm-overview}
\end{figure}

To reduce reliance on the interaction history at inference time, alternative approaches encode it into parameters, adapters, or internal states along two directions. Context compilation methods, such as Doc-to-LoRA~\citep{charakorn2026doctoloralearninginstantlyinternalize} and SHINE~\citep{liu2026shinescalableincontexthypernetwork}, use hypernetworks to map individual contexts to LoRA adapters. State maintenance methods integrate information sequentially. RPMem~\citep{zhao2026rpmemlearninglongtermrecurrent} recurrently consolidates session representations into memory adapters, while Metis~\citep{zhang2026metismemoryfoundationmodel} updates memory matrices within the backbone.

Despite advances in encoding interaction history and updating memory across sessions, two limitations remain. First, independently compiling segments into adapters does not itself support selective retention and revision as new sessions arrive~\citep{charakorn2026doctoloralearninginstantlyinternalize}. Second, recurrent consolidation enables revision, but successive updates to a shared state can attenuate earlier contributions~\citep{zhao2026rpmemlearninglongtermrecurrent,yang2025gateddeltanetworksimproving}. These limitations motivate preserving direct access to earlier evidence while supporting selective revision as the interaction history grows.

To address these challenges, we propose Dual-Path Parametric Memory (DPPM), which separates evidence accumulation from sequential state updates within a unified parametric memory interface. The Evidence path directly aggregates representations of the interaction history, while the Delta path maintains an associative state in chronological order. Their outputs jointly form the final memory, combining accumulated evidence with order-dependent updates.

Specifically, DPPM learns cross-session memory consolidation on top of a frozen Doc-to-LoRA compiler, avoiding additional compiler training. The compiler encodes chronological segments of the interaction history into representations shared by both paths. In the Evidence path, coordinate-wise, content-dependent weighting directly aggregates these representations, allowing early evidence to contribute without passing through successive state updates. In the Delta path, the same segment representations are processed in chronological order. A content-derived associative key retrieves a prediction of the current segment representation from the existing state, and the prediction residual is written back to capture order-dependent changes. Finally, learnable fusion and scale calibration combine the two outputs, and a frozen decoder converts the fused memory into history-conditioned LoRA adapters for personalized generation. Training updates only the lightweight memory modules, while the compiler, decoder, and language-model backbone remain frozen.

We evaluate DPPM on PersonaMem-v2~\citep{jiang2025personamemv2personalizedintelligencelearning} and PrefEval~\citep{zhao2025llmsrecognizepreferencesevaluating} across three backbones spanning different architectures and model sizes (2B--7B parameters). With Qwen3-4B, DPPM outperforms the evaluated baselines, achieving \textbf{54.22\%} accuracy on PersonaMem-v2 and \textbf{86.79\%} accuracy on PrefEval. These results suggest that DPPM offers a simple and effective approach to cross-session personalized parametric memory. In summary, this work makes the following three key contributions:
\begin{enumerate}[label=(\arabic*),wide=0pt,nosep,
  before={\clubpenalty=10000\widowpenalty=10000}]
  \item We propose DPPM, separating evidence accumulation from sequential updates for cross-session memory without compiler retraining.
  \item We analyze how direct accumulation mitigates evidence attenuation and sequential updates capture order-dependent changes.
  \item We validate DPPM across benchmarks, model architectures, and scales through ablations and controlled experiments.
\end{enumerate}

\section{Preliminary}
\label{sec:preliminary}

\subsection{Parametric Memory Workflow}

Let $\mathcal{C}_{1:T}=(c_1,\ldots,c_T)$ denote the interaction history, divided into $T$ chronological segments. Each segment $c_t$ is a contiguous portion of the interaction history. The consolidated memory is
\begin{equation}
  M_T=U(m_1,\ldots,m_T),
  \label{eq:memory-workflow}
\end{equation}
where $m_t=E(c_t)$ is the compiled memory for segment $c_t$, and $U$ aggregates these memories in chronological order. A decoder produces memory parameters $\Lambda_T=D(M_T)$. Let $\theta$ denote the frozen backbone parameters and $\oplus$ the application of adapter updates, giving $\theta_T=\theta\oplus\Lambda_T$. For a new query $x$, the response distribution is
\begin{equation}
  p_{\theta_T}(y\mid x)=\prod_{j=1}^{|y|}p_{\theta_T}(y_j\mid x,y_{<j}).
  \label{eq:memory-inference}
\end{equation}

Here $y_j$ is the $j$th response token and $y_{<j}$ its prefix. Inference uses compiled memory without the interaction history in the prompt.

\subsection{Hypernetwork-Generated LoRA}

Following Doc-to-LoRA~\citep{charakorn2026doctoloralearninginstantlyinternalize}, the hypernetwork $H_\phi=D\circ E$, parameterized by $\phi$, comprises a context encoder $E$, which maps text to latent memory, and an adapter decoder $D$, which generates LoRA factors. For a segment $c$ of the interaction history and target matrix $W_\ell$, indexed by $\ell$, these factors are
\begin{equation}
  (A_\ell(c),B_\ell(c))=H_\phi^{(\ell)}(c).
  \label{eq:hypernetwork-lora}
\end{equation}

They adapt the frozen weight through
\begin{equation}
  W_\ell(c)=W_\ell+s_\ell B_\ell(c)A_\ell(c).
  \label{eq:lora-update}
\end{equation}

Here $r_{\mathrm L}$ is the LoRA rank, $A_\ell\in\mathbb{R}^{r_{\mathrm L}\times d_{\mathrm{in}}}$, and $B_\ell\in\mathbb{R}^{d_{\mathrm{out}}\times r_{\mathrm L}}$. The dimensions $d_{\mathrm{in}}$ and $d_{\mathrm{out}}$ are the input and output widths of $W_\ell$, and $s_\ell$ scales the update. Once trained, the shared hypernetwork generates adapters in a forward pass without per-segment optimization.

\begin{figure*}[t]
  \centering
  \includegraphics[width=\textwidth]{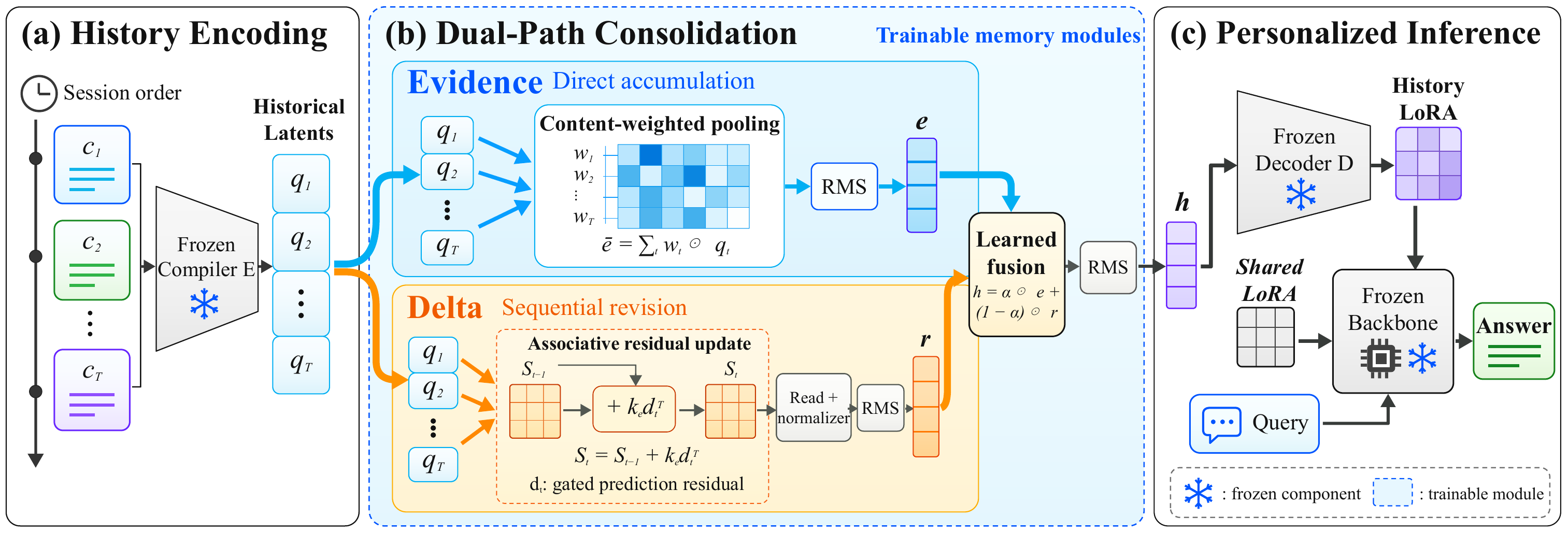}
  \caption{Overall architecture of DPPM. A frozen compiler encodes the interaction history segment by segment, and DPPM consolidates the resulting representations through Evidence pooling and sequential Delta updates. Their calibrated outputs are fused and decoded into history-conditioned LoRA adapters. }
  \label{fig:dppm-framework}
\end{figure*}
\section{Method}
\label{sec:method}

\subsection{Overview and Setup}
Using the interaction history $\mathcal{C}_{1:T}$ defined in Section~\ref{sec:preliminary}, we encode each segment $c_t$ with the frozen compiler $E$. At a fixed position in the compiled memory $m_t$, let $q_t\in\mathbb{R}^d$ denote the segment representation, where $d$ is the latent width. We share consolidation parameters across positions and exclude the query and answer from the history.

DPPM instantiates the aggregation operator in Section~\ref{sec:preliminary} with two paths (Figure~\ref{fig:dppm-framework}). Evidence directly pools segment representations, while Delta updates an associative state in chronological order. Their fused output $h$ is decoded into a history-conditioned adapter. We first motivate the two paths, then define their computations.

\subsection{Motivation for Dual-Path Memory}
For the same interaction history $\mathcal{C}_{1:T}$, consider a recurrent memory. Its state after segment $t$ is $\mu_t=F_t(\mu_{t-1},q_t)$, where $F_t$ updates the previous state using the segment representation $q_t$. Compare two interaction histories that differ at an earlier segment $i$ but share all later segments, and write $\Delta \mu_t=\mu_t-\mu'_t$ for their state difference. If every subsequent update contracts this difference by at most a factor $0\leq\kappa<1$, then
\begin{equation}
 \|\Delta \mu_T\|_2\leq\kappa^{T-i}\|\Delta \mu_i\|_2,
 \label{eq:motivation-recurrence}
\end{equation}
where $\|\cdot\|_2$ is the Euclidean norm and $T-i$ is the age of the earlier contribution. Thus, its influence may attenuate through repeated updates. Proposition~\ref{prop:recurrent-contraction} gives sufficient conditions on the full update Jacobian, including state-dependent gates.

Evidence instead connects each segment representation directly to the pooled memory. Let $w_{i,j}$ be segment $i$'s pooling weight for latent coordinate $j$. If the largest and smallest scores across segments differ by at most a constant $B$ that does not grow with $T$, then
\begin{equation}
 w_{i,j}\geq[1+(T-1)e^B]^{-1}.
 \label{eq:motivation-evidence}
\end{equation}

For changes to $q_i$ that leave its scores unchanged, this also bounds how much the raw pooled vector changes relative to the input change (Proposition~\ref{prop:direct-evidence}). This lower bound depends on history length $T$ but not segment position $i$, so earlier and later segments share the same bound. Since it still decreases with $T$, direct accumulation is motivated as a way to mitigate age-related attenuation, rather than eliminate forgetting.

However, Evidence pooling gives the same result when fixed segment representations are reordered, so it cannot distinguish opposite update orders requiring different answers (Proposition~\ref{prop:pooling-order}). Delta supplies ordered revision.

\subsection{Evidence Accumulation}
To provide a direct route from the interaction history to memory, Evidence weights and pools the segment representations. Using normalized representations $z_t$ of the segment $c_t$, we apply a learned projection $W_E\in\mathbb{R}^{d\times d}$ and coordinate-wise temporal softmax to obtain
{\setlength{\jot}{2pt}
\begin{gather}
 s_t=W_Ez_t,\quad w_{t,j}=\frac{\exp(s_{t,j})}{\sum_{\tau=1}^T\exp(s_{\tau,j})},
 \label{eq:dppm-evidence}\\
 \bar e=\sum_{t=1}^T w_t\odot q_t.
 \label{eq:dppm-pooling}
\end{gather}}

Here $s_t,w_t\in\mathbb{R}^d$ are score and weight vectors, and $\odot$ denotes element-wise multiplication. Each coordinate $j$ pools segment representations using scores from normalized representations.

To compensate for magnitude changes from pooling, we rescale $\bar e$ into the calibrated Evidence representation $e$. The target root-mean-square (RMS) scale $\rho_E$ is derived from the coordinate-wise weighted second moment of the segment representations. This rescaling, denoted by $\mathcal{R}$, uses a small positive floor for numerical stability and also calibrates the Delta and fused representations. Appendix~\ref{app:calibration} gives the definitions.

\subsection{Sequential Delta Updates}
To complement Evidence with ordered revision, Delta processes the same segment representations sequentially. For key dimension $K$, it computes an unnormalized key $p_t\in\mathbb{R}^K$, a unit key $k_t$, and a coordinate-wise write gate $\beta_t\in(0,1)^d$,
\begin{equation}
 p_t=\operatorname{ELU}(W_Kz_t)+\mathbf{1}_K,\quad k_t=\frac{p_t}{\|p_t\|_2},
 \label{eq:dppm-key}
\end{equation}
\begin{equation}
 \beta_t=\sigma(W_\beta z_t+b_\beta),
 \label{eq:dppm-write-gate}
\end{equation}

Here ELU is the element-wise exponential linear unit with unit parameter, $\sigma$ is the sigmoid, and $\mathbf{1}_n$ denotes an $n$-dimensional all-ones vector. The learned projections have shapes $W_K\in\mathbb{R}^{K\times d}$ and $W_\beta\in\mathbb{R}^{d\times d}$, with bias $b_\beta\in\mathbb{R}^d$. Using this key and gate, Delta reads a prediction $\hat q_t\in\mathbb{R}^d$ from the associative state $S_{t-1}\in\mathbb{R}^{K\times d}$ and writes back a gated residual $d_t\in\mathbb{R}^d$,
\begin{equation}
 \hat q_t=S_{t-1}^\top k_t,
 \label{eq:dppm-prediction}
\end{equation}
\begin{equation}
 d_t=\beta_t\odot(q_t-\hat q_t),
 \label{eq:dppm-residual}
\end{equation}
\begin{equation}
 S_t=S_{t-1}+k_td_t^\top.
 \label{eq:dppm-delta}
\end{equation}
The delta-rule updates~\citep{schlag2021lineartransformerssecretlyfast,yang2025gateddeltanetworksimproving} correct the current key's prediction. At another key $k'\in\mathbb{R}^K$, the readout changes by $(k_t^\top k')d_t$. Key similarity controls interference.

To stabilize the readout scale, a parallel normalizer $N_t\in\mathbb{R}^{K\times d}$ follows the same updates as $S_t$, using an all-ones target. After processing the interaction history, a learned read key $a=\operatorname{softmax}(u)$, with $u\in\mathbb{R}^K$, gives the normalized readout
\begin{equation}
 \bar r=(S_T^\top a)\oslash[N_T^\top a]_\delta,
 \label{eq:dppm-readout}
\end{equation}
where $\oslash$ denotes element-wise division and $[\cdot]_\delta$ applies a small positive floor to the denominator. RMS calibration then produces the Delta representation $r$. Appendix~\ref{app:delta-normalization} details the normalization, initialization, and calibration.

\subsection{Fusion and Adapter Construction}
We fuse Evidence $e$ and Delta $r$ using learned logits $v\in\mathbb{R}^d$ and weights $\alpha\in(0,1)^d$,
\begin{equation}
 \alpha=\sigma(v),\quad \bar h=\alpha\odot e+(\mathbf{1}_d-\alpha)\odot r.
 \label{eq:dppm-fusion}
\end{equation}
RMS calibration then produces $h$ using a target scale derived from both paths (Appendix~\ref{app:calibration}). Both $\alpha$ and the read key $a$ remain fixed after training.

Finally, the frozen decoder generates LoRA factors $(A_\ell^H,B_\ell^H)=D^{(\ell)}(h)$, where $D^{(\ell)}$ denotes its output for target matrix $W_\ell$. Following Equation~\ref{eq:lora-update}, the serving weight is
\begin{equation}
 W_\ell^H=W_\ell+s_\ell B_\ell^H A_\ell^H+\Delta W_{\ell,\mathrm{shared}}.
 \label{eq:dppm-adapter-weight}
\end{equation}
The product $B_\ell^H A_\ell^H$ forms the history-conditioned low-rank update. The frozen shared-adapter update $\Delta W_{\ell,\mathrm{shared}}$ is added once.

\subsection{Answer Supervision}
We train both paths and their fusion with cross-entropy over answers. Gradients propagate through the decoder and backbone to update only the memory modules. The compiler, decoder, shared adapter, and backbone remain frozen.

\section{Experiments}
\subsection{Experimental Setup}
\label{sec:experimental-setup}

\paragraph{Dataset.}
Following RPMem~\citep{zhao2026rpmemlearninglongtermrecurrent}, we use the processed PersonaMem-v2 32K-Text split~\citep{jiang2025personamemv2personalizedintelligencelearning} with 18,527 training, 2,059 validation, and 5,000 test examples, excluding test personas from training and validation. PrefEval~\citep{zhao2025llmsrecognizepreferencesevaluating} uses a fixed topic split (seed 42), with 16 training and four held-out topics. Three preference forms at 10, 70, and 300 turns yield 7,380 training and 1,620 test instances.

\paragraph{Baselines.}
We compare No Context and Full Context, textual memories from Rolling Summary following RPMem~\citep{zhao2026rpmemlearninglongtermrecurrent}, top-10 BGE-M3 retrieval~\citep{chen2025m3embeddingmultilingualitymultifunctionalitymultigranularity}, Mem0~\citep{chhikara2025mem0buildingproductionreadyai}, and LightMem~\citep{fang2026lightmemlightweightefficientmemoryaugmented}. Parametric baselines include D2L~\citep{charakorn2026doctoloralearninginstantlyinternalize}, PLUME~\citep{shi2026evolvingcontextparameterizationlarge}, RPMem~\citep{zhao2026rpmemlearninglongtermrecurrent}, $\delta$-mem~\citep{lei2026deltamemefficientonlinememory}, and Metis~\citep{zhang2026metismemoryfoundationmodel}. For a fair comparison, our RPMem baseline uses the same frozen D2L encoder as DPPM, without additional encoder-training data. We use released checkpoints for $\delta$-mem and Metis. Gold State provides annotated user state as a reference.

\begin{table*}[!t]
  \centering
  \small
  \setlength{\tabcolsep}{3pt}
  \begin{tabular*}{\textwidth}{@{\extracolsep{\fill}}llrrrrrrrr@{}}
\toprule
\multirow{2}{*}{\textbf{Type}} & \multirow{2}{*}{\textbf{Method}} & \multicolumn{3}{c}{\textbf{PersonaMem-v2}} & \multicolumn{4}{c}{\textbf{PrefEval}}  & \multirow{2}{*}{\textbf{Avg.}} \\
\cmidrule(lr){3-5}\cmidrule(lr){6-9}
 & & Overall & Self & Current & Overall & 10 & 70 & 300  & \\
\midrule
Reference & No Context & 28.54 & 29.83 & 30.38 & 37.72 & 37.04 & 39.26 & 36.85 & 33.13 \\
 & Gold State & 61.08 & 67.44 & 67.44 & 93.52 & 92.78 & 94.07 & 93.70 & 77.30 \\
\midrule
Textual & Rolling Summary & 29.16 & 30.93 & 31.12 & 40.49 & 38.70 & 40.00 & 42.78 & 34.83 \\
 & LightMem & 30.66 & 31.76 & 32.46 & 71.60 & \underline{80.37} & 69.63 & 64.81 & 51.13 \\
 & Mem0 & 30.46 & 31.93 & 33.06 & 75.49 & 77.96 & 76.11 & 72.41 & 52.98 \\
 & RAG (BGE-M3) & 32.50 & 34.75 & 34.40 & 57.96 & 65.56 & 57.59 & 50.74 & 45.23 \\
 & Full Context & 32.96 & 35.89 & 36.68 & 45.31 & 51.11 & 45.74 & 39.07 & 39.14 \\
\midrule
Parametric & Direct D2L & 24.86 & 25.15 & 24.92 & 30.06 & 40.19 & 25.56 & 24.44 & 27.46 \\
 & PLUME & 30.98 & 32.80 & 34.63 & 40.12 & 45.74 & 39.63 & 35.00 & 35.55 \\
 & $\delta$-mem & 30.58 & 32.72 & 34.18 & 37.22 & 36.11 & 37.96 & 37.59 & 33.90 \\
 & Metis-4B & 30.86 & 32.05 & 32.56 & 38.64 & 40.00 & 40.00 & 35.93 & 34.75 \\
 & Metis-9B & 29.14 & 28.45 & 30.05 & 38.09 & 36.48 & 39.63 & 38.15 & 33.62 \\
 & Metis-27B & 34.68 & 36.62 & 38.35 & 43.21 & 44.26 & 42.41 & 42.96 & 38.95 \\
 & RPMem & \underline{45.73} & \underline{49.46} & \underline{53.38} & \underline{82.47} & 79.56 & \underline{83.81} & \underline{84.04} & \underline{64.10} \\
 & \textbf{DPPM (Ours)} & \textbf{54.22} & \textbf{54.69} & \textbf{58.38} & \textbf{86.79} & \textbf{84.52} & \textbf{88.15} & \textbf{87.70} & \textbf{70.51} \\
\bottomrule
\end{tabular*}

  \caption{Main results on both benchmarks (accuracy, \%). Self and Current select \texttt{who=self} and \texttt{updated=False}. PrefEval interval scores average three preference forms. Avg. averages the two Overall scores. \textbf{Bold} and \underline{underlining} indicate the best and second-best scores, excluding the annotated Gold State reference. Metis uses its native Qwen3.5 backbone at the indicated size; the other methods use Qwen3-4B.}
  \label{tab:main-results}
\end{table*}

\paragraph{Hyperparameters.}
We use Qwen3-4B-Instruct-2507, Gemma-2-2B-IT, and Mistral-7B-Instruct-v0.2 with released D2L compilers. DPPM is trained separately per dataset for five epochs using AdamW, learning rate $10^{-3}$, weight decay $0.01$, batch size 1, and gradient clipping at 1. We evaluate the final checkpoint of each run; further settings appear in Appendix~\ref{app:configuration}.

\subsection{Main Results}
\label{sec:main-results}

\begin{table*}[!t]
  \centering
  \small
  \setlength{\tabcolsep}{2pt}
  \begin{tabular*}{\textwidth}{@{\extracolsep{\fill}}llrrrrrrr@{}}
\toprule
\multirow{2}{*}{\textbf{Backbone}} & \multirow{2}{*}{\textbf{Method}} & \multicolumn{3}{c}{\textbf{PersonaMem-v2}} & \multicolumn{4}{c}{\textbf{PrefEval}} \\
\cmidrule(lr){3-5}\cmidrule(lr){6-9}
 & & Overall & Self & Current & Overall & 10 & 70 & 300 \\
\midrule
\multirow{4}{*}{\shortstack[l]{Gemma-2-2B\\IT}} & RPMem & 29.17$_{\pm0.34}$ & 30.28$_{\pm0.40}$ & 30.88$_{\pm0.44}$ & 75.01$_{\pm0.92}$ & 73.07$_{\pm1.70}$ & 75.70$_{\pm1.03}$ & 76.26$_{\pm1.30}$ \\
 & Evidence only & \underline{38.20}$_{\pm0.23}$ & \underline{41.34}$_{\pm0.24}$ & \underline{44.50}$_{\pm0.24}$ & \textbf{79.81}$_{\pm0.58}$ & 77.52$_{\pm0.71}$ & \textbf{81.33}$_{\pm1.19}$ & \textbf{80.59}$_{\pm1.25}$ \\
 & Delta only & 37.68$_{\pm0.69}$ & 40.96$_{\pm0.83}$ & 43.13$_{\pm0.95}$ & \underline{79.38}$_{\pm1.24}$ & \underline{77.74}$_{\pm1.39}$ & \underline{80.15}$_{\pm1.09}$ & \underline{80.26}$_{\pm1.69}$ \\
 & \textbf{DPPM} & \textbf{40.31}$_{\pm0.54}$ & \textbf{43.71}$_{\pm0.63}$ & \textbf{46.77}$_{\pm0.79}$ & 79.20$_{\pm0.76}$ & \textbf{78.56}$_{\pm0.44}$ & 79.56$_{\pm0.59}$ & 79.48$_{\pm2.60}$ \\
\midrule
\multirow{4}{*}{\shortstack[l]{Qwen3-4B\\Instruct-2507}} & RPMem & 45.73$_{\pm0.24}$ & 49.46$_{\pm0.35}$ & 53.38$_{\pm0.35}$ & 82.47$_{\pm0.58}$ & 79.56$_{\pm1.10}$ & 83.81$_{\pm1.19}$ & 84.04$_{\pm1.27}$ \\
 & Evidence only & 48.09$_{\pm0.63}$ & 50.39$_{\pm0.70}$ & 54.06$_{\pm0.85}$ & 84.59$_{\pm0.74}$ & 81.67$_{\pm1.02}$ & 86.00$_{\pm1.89}$ & 86.11$_{\pm1.71}$ \\
 & Delta only & \underline{49.77}$_{\pm0.93}$ & \underline{51.16}$_{\pm1.32}$ & \underline{54.80}$_{\pm1.22}$ & \underline{85.98}$_{\pm0.84}$ & \underline{84.48}$_{\pm0.24}$ & \underline{87.00}$_{\pm2.05}$ & \underline{86.44}$_{\pm0.88}$ \\
 & \textbf{DPPM} & \textbf{54.22}$_{\pm0.48}$ & \textbf{54.69}$_{\pm0.88}$ & \textbf{58.38}$_{\pm0.94}$ & \textbf{86.79}$_{\pm0.68}$ & \textbf{84.52}$_{\pm0.89}$ & \textbf{88.15}$_{\pm1.01}$ & \textbf{87.70}$_{\pm0.64}$ \\
\midrule
\multirow{4}{*}{\shortstack[l]{Mistral-7B\\Instruct-v0.2}} & RPMem & 31.24$_{\pm0.61}$ & 32.08$_{\pm0.71}$ & 33.32$_{\pm0.64}$ & 79.53$_{\pm0.32}$ & 79.52$_{\pm0.83}$ & 79.26$_{\pm0.57}$ & 79.81$_{\pm1.23}$ \\
 & Evidence only & \underline{42.29}$_{\pm0.27}$ & \underline{45.08}$_{\pm0.35}$ & \underline{48.65}$_{\pm0.34}$ & 82.65$_{\pm0.89}$ & 81.41$_{\pm1.04}$ & 83.37$_{\pm1.86}$ & 83.19$_{\pm1.73}$ \\
 & Delta only & 39.98$_{\pm0.36}$ & 42.02$_{\pm0.65}$ & 45.19$_{\pm0.69}$ & \underline{84.49}$_{\pm1.19}$ & \underline{82.59}$_{\pm1.02}$ & \underline{85.04}$_{\pm1.19}$ & \underline{85.85}$_{\pm2.09}$ \\
 & \textbf{DPPM} & \textbf{45.22}$_{\pm0.33}$ & \textbf{47.78}$_{\pm0.46}$ & \textbf{51.42}$_{\pm0.50}$ & \textbf{84.84}$_{\pm1.11}$ & \textbf{83.26}$_{\pm0.86}$ & \textbf{85.26}$_{\pm1.06}$ & \textbf{86.00}$_{\pm1.63}$ \\
\bottomrule
\end{tabular*}

  \caption{Memory alternatives across three backbones (accuracy, \%). All entries are five-run means, with sample standard deviations in subscripts. Best and second-best means are bold and underlined within each backbone.}
  \label{tab:cross-backbone}
\end{table*}

Table~\ref{tab:main-results} shows that DPPM reaches \textbf{54.22\%} on PersonaMem-v2 and \textbf{86.79\%} on PrefEval, improving over RPMem by 8.50 and 4.32 percentage points. The improvement extends across all seven reported metrics, including both information about the focal user and the currently valid state. On PrefEval, the advantage remains positive at every dialogue interval, including a 3.67-point gain after 300 intervening turns. Thus, the overall result reflects benefits across several evaluation conditions rather than a single favorable subset.

The strongest textual baseline differs across benchmarks, with Full Context leading on PersonaMem-v2 and Mem0 on PrefEval. DPPM exceeds their Overall scores by 21.26 and 11.30 points, respectively. Providing or retrieving text from the interaction history therefore leaves substantial room for better personalization in these settings. Direct D2L also trails the trained consolidation methods, suggesting that compiling the interaction history into parameters benefits from learning how its segments should be integrated for downstream answers. The significance comparison with RPMem appears in Appendix~\ref{app:significance}.

\subsection{Ablation Study}
\paragraph{Cross-backbone architectural comparison.}
\label{sec:backbone-results}

Table~\ref{tab:cross-backbone} compares the four memory designs within each backbone and its released compiler. DPPM outperforms RPMem on both benchmarks for Gemma, Qwen, and Mistral. The gains range from 8.50 to 13.98 points on PersonaMem-v2 and from 4.19 to 5.31 points on PrefEval. The largest gains occur with Mistral, while Gemma confirms gains with a smaller backbone.

Architectural ablations use Evidence only and Delta only as single-path controls to examine the complementarity of the two memory mechanisms. Their relative strengths vary across tasks and backbones, indicating that direct evidence accumulation and sequential state updates provide distinct memory signals. Combining both paths improves PersonaMem-v2 Overall over either standalone variant on all three backbones and PrefEval Overall on Qwen and Mistral. On Qwen, DPPM exceeds the stronger standalone path by 4.45 and 0.81 percentage points, respectively. These results support combining complementary memory signals to enhance cross-session personalization.

\paragraph{Parameter-matched comparison.}
\label{sec:parameter-matching}
To examine whether DPPM's gains arise from a larger trainable budget, we compare the memory modules under the same frozen Qwen3-4B compiler--decoder interface. Original RPMem uses 524,800 trainable parameters, while DPPM uses only 2,564 more (0.49\%). We also evaluate an RPMem control matched at 527,364 parameters.

The control retains RPMem's recurrence, adding a rank-4 nonlinear gate correction with a biased $512\!\to\!4$ projection (2,052 parameters), a fixed orthonormal DCT basis, and 512 learned coordinate gains. Zero-initialized gains preserve the initial gate output; all added parameters are trained. All runs share data splits, frozen components, CE, optimizer settings, five epochs, seeds 42--46, and final-epoch evaluation.

\begin{table}[htbp]
  \centering
  \small
  \setlength{\tabcolsep}{2pt}
  \begin{tabular*}{\columnwidth}{@{\extracolsep{\fill}}lrrr@{}}
\toprule
\textbf{Memory} & \textbf{Params} & \textbf{Persona} & \textbf{PrefEval} \\
\midrule
RPMem & 524800 & \underline{45.73}$_{\pm0.24}$ & 82.47$_{\pm0.58}$ \\
RPMem-matched & 527364 & 45.49$_{\pm0.42}$ & \underline{82.56}$_{\pm0.90}$ \\
\textbf{DPPM} & 527364 & \textbf{54.22}$_{\pm0.48}$ & \textbf{86.79}$_{\pm0.68}$ \\
\bottomrule
\end{tabular*}

  \caption{Scores are five-run Overall accuracy means (\%), with sample standard deviations in subscripts. Persona denotes PersonaMem-v2.}
  \label{tab:parameter-comparison}
\end{table}

As shown in Table~\ref{tab:parameter-comparison}, increasing RPMem's budget to match DPPM yields 45.49\% on PersonaMem-v2 and 82.56\% on PrefEval, close to the original gate's results. With the same parameter count, DPPM improves these scores by 8.73 and 4.23 percentage points. This comparison supports dual-path consolidation beyond the small difference in trainable capacity.

\subsection{Further Analysis}
\label{sec:further-analysis}
\paragraph{Accuracy, latency, and memory.}

Figure~\ref{fig:latency}(a) shows that DPPM improves accuracy over RPMem by 8.50 percentage points at an answer-forward latency of 46.14 ms, only 2.43 ms higher. It also achieves higher accuracy than RAG and Full Context with approximately $3.1\times$ and $62.7\times$ faster answer computation, respectively.

To examine memory costs with longer histories, Figure~\ref{fig:latency}(b) uses one question selected by maximum official 128K history length, evaluated with both official interaction history versions. Full Context inputs grow from 33,143 to 139,319 tokens, increasing peak GPU memory from 10,203.91 to 18,245.97 MiB. DPPM remains at 7,762.32 MiB, matching RPMem and using 57.46\% less memory than Full Context for the longer history. Together, these results show improved personalization with low answer latency and a compact answering footprint after memory preparation.

\begin{figure*}[!t]
  \centering
  \includegraphics[width=\textwidth]{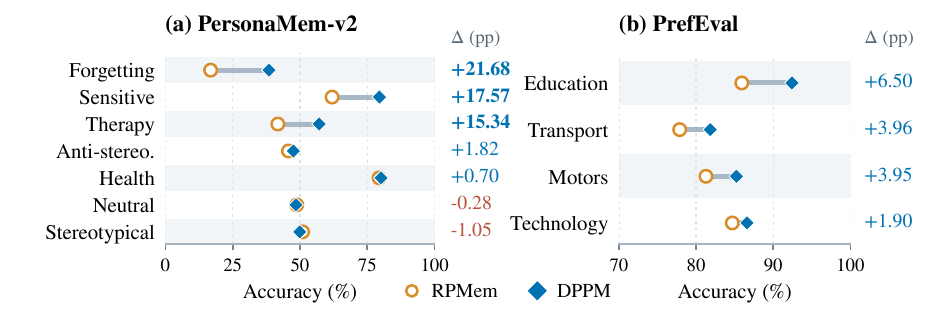}
  \caption{Category-level comparison of RPMem and DPPM. Circles and diamonds show five-run mean accuracies, and the right-hand annotations give DPPM's gain in percentage points. The panels cover all seven PersonaMem-v2 information types and all four held-out PrefEval topics. Axis ranges are shown separately for each benchmark.}
  \label{fig:category-gains}
\end{figure*}
\begin{figure}[!t]
  \centering
  \includegraphics[width=\linewidth]{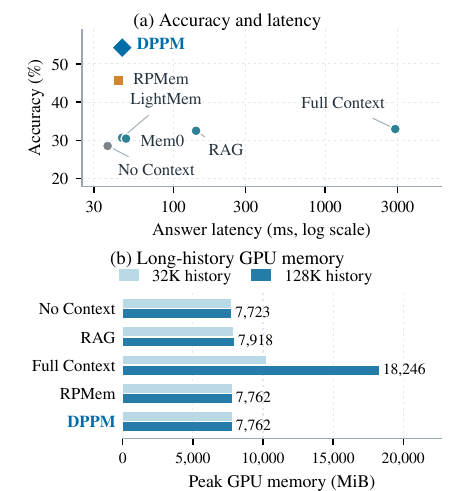}
  \caption{Answer-stage efficiency on Qwen3-4B with one A800 (batch size 1). (a) Full-test accuracy versus weighted latency. (b) Peak GPU memory for paired histories of one length-selected question. Results use three-forward medians after one warmup, excluding compilation and retrieval; KV cache is disabled in (b).}
  \label{fig:latency}
\end{figure}

\paragraph{Sources of improvement.}
\label{sec:category-results}

Figure~\ref{fig:category-gains} links DPPM's gains to the appropriate use of accumulated evidence. In PersonaMem-v2, Forget tests whether responses respect requests to stop using prior information, Sensitive tests restraint in using private details, and Therapy tests personalization from individual background~\citep{jiang2025personamemv2personalizedintelligencelearning}. DPPM improves these categories over RPMem by 21.68, 17.57, and 15.34 points, respectively. The Others subset tests whether information about another person is incorrectly attributed to the user; accuracy rises from 13.68\% to 50.23\%. These results indicate that much of DPPM's improvement comes from better attribution and more appropriate revision of memory use.

This pattern aligns with our motivation to preserve evidence while revising its applicability. The Evidence path gives earlier representations direct access to the final memory, while Delta supplies ordered updates as new interactions arrive. DPPM also outperforms both single-path controls on Others and Forget, supporting their complementary roles. On PrefEval, the four topics test preference following in different application domains~\citep{zhao2025llmsrecognizepreferencesevaluating}. Gains across all four, from 1.90 points on Technology to 6.50 on Education, show broader use of remembered preferences. Alongside the latent-retention results below, these findings support combining evidence accumulation and state revision for personalization.

\paragraph{Latent retention over time.}
To test whether direct evidence accumulation mitigates the attenuation of earlier information, we measure latent sensitivity as subsequent interactions accumulate. We freeze the seed-42 benchmark checkpoints and move one preference among 32 fixed distractors after a shared neutral initial segment, keeping each interaction history at 34 segments. For paired preference values A/B, latent sensitivity is $\|h_A-h_B\|_2/\|q_A-q_B\|_2$, where $h$ is the consolidated representation and $q$ the representation of the preference segment. We normalize each profile's sensitivity by its age-zero value and average over 24 profiles. Figure~\ref{fig:latent-retention} shows that at age 32, DPPM retains 0.970/0.667 for PersonaMem-v2/PrefEval checkpoints, compared with RPMem's 0.268/0.116 and Delta only's 0.367/0.300. Evidence only stays at 1.000. This supports a less age-sensitive latent response from the direct evidence route.

\begin{figure}[!t]
  \centering
  \includegraphics[width=\linewidth]{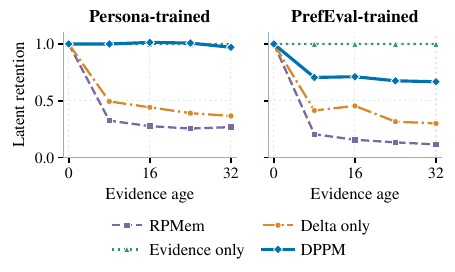}
  \caption{Latent retention versus the number of subsequent distractor segments.  Curves average 24 profiles after within-profile age-zero normalization.}
  \label{fig:latent-retention}
\end{figure}

\paragraph{Training objectives.}
\label{sec:objective-results}

To assess compatibility with different objectives, we fix DPPM's architecture and compare CE with distillation and reward-based variants (Table~\ref{tab:training-objectives}). Answer FKL matches teacher answer distributions. On-policy distillation (OPD)~\citep{agarwal2024onpolicydistillationlanguagemodels} applies token-level reverse KL to student rollouts capped at 16 or 64 tokens, using a frozen teacher conditioned on training-only user state. One-step GRPO~\citep{shao2024deepseekmathpushinglimitsmathematical} uses correctness rewards for sampled answers. All variants retain CE.

The short OPD recipe and answer FKL score below CE on both datasets. The 64-token OPD configuration, which also uses the revised teacher prompt, reaches 54.56\% and 87.04\%, giving the highest two-dataset average of 70.80\%. GRPO attains the highest PersonaMem-v2 score but trades off PrefEval accuracy, reducing its average below CE. These results suggest that appropriately configured on-policy feedback can complement the dual-path architecture, while plain CE remains competitive with simpler training.

\begin{table}[!t]
  \centering
  \small
  \setlength{\tabcolsep}{3pt}
  \begin{tabular*}{\columnwidth}{@{\extracolsep{\fill}}lrrr@{}}
\toprule
\textbf{Objective} & \textbf{Persona} & \textbf{PrefEval} & \textbf{Avg.} \\
\midrule
CE (DPPM) & 54.22 & \underline{86.79} & \underline{70.51} \\
+ Answer FKL & 53.98 & 86.05 & 70.02 \\
+ OPD (16) & 52.96 & 85.80 & 69.38 \\
+ OPD (64) & \underline{54.56} & \textbf{87.04} & \textbf{70.80} \\
+ GRPO & \textbf{54.80} & 85.06 & 69.93 \\
\bottomrule
\end{tabular*}

  \caption{Training objectives on DPPM (accuracy, \%). Persona denotes PersonaMem-v2. OPD (16/64) caps student rollouts at 16/64 tokens.}
  \label{tab:training-objectives}
\end{table}

\section{Related Work}

\subsection{Personalized Memory}

Personalized memory must distill growing histories into useful information for future responses. MemoryBank and MemGPT maintain persistent records and retrieve relevant content when needed~\citep{zhong2023memorybankenhancinglargelanguage,packer2024memgptllmsoperatingsystems}. As records accumulate, storage must be complemented by content selection and revision. Mem0 and A-MEM address these needs through memory extraction, organization, and updating~\citep{chhikara2025mem0buildingproductionreadyai,xu2025amemagenticmemoryllm}. Recent work learns adaptive memory strategies that balance shared knowledge and individual needs~\citep{xu-etal-2026-learning,han2026learningsharepersonalizehierarchical}.

These developments broaden evaluation from recalling history to using it appropriately. LoCoMo and LongMemEval assess conversational recall and reasoning~\citep{maharana2024evaluatinglongtermconversationalmemory,wu2025longmemevalbenchmarkingchatassistants}, while PrefEval and PersonaMem-v2 examine preference following and implicit personalization~\citep{zhao2025llmsrecognizepreferencesevaluating,jiang2025personamemv2personalizedintelligencelearning}. Evolving preferences and multi-user decisions introduce further demands on memory maintenance~\citep{liu2026permabenchmarkingpersonalizedmemory,chen2026vehiclemembenchexecutablebenchmarkmultiuser}. We approach these demands through parametric memory for cross-session personalization.

\subsection{Parametric Memory}

Encoding information into parameters reduces reliance on interaction history at inference time, but requires an effective writing mechanism. ROME and MEMIT edit factual associations~\citep{meng2023locatingeditingfactualassociations,meng2023masseditingmemorytransformer}, while FastMem internalizes prompts through optimization~\citep{zhu2024fastmemfastmemorizationprompt}. To avoid repeated optimization, Doc-to-LoRA and SHINE learn context-to-adapter hypernetworks that generate LoRA parameters in a forward pass~\citep{charakorn2026doctoloralearninginstantlyinternalize,liu2026shinescalableincontexthypernetwork}.

Efficient writing alone does not determine how accumulated memory should evolve. Sequential state updates address this requirement through associative rules in Fast Weight Programmers and Gated DeltaNet~\citep{schlag2021lineartransformerssecretlyfast,yang2025gateddeltanetworksimproving}, and learned adaptation in TTT and Titans~\citep{sun2025learninglearntesttime,behrouz2024titanslearningmemorizetest}. Building on state maintenance, Metis integrates forward-updated memory into the backbone~\citep{zhang2026metismemoryfoundationmodel}. RPMem, the work closest to ours, connects recurrent session consolidation with adapter generation~\citep{zhao2026rpmemlearninglongtermrecurrent}. DPPM extends this approach with direct evidence accumulation alongside sequential revision, reducing reliance on a single recurrent state.

\FloatBarrier
\section{Conclusion}
We introduced DPPM for integrating interaction history across sessions in personalized language models. DPPM combines evidence accumulation and sequential updates for retention and revision without extra compiler training. Experiments across two benchmarks and backbones of different architectures and sizes demonstrate its effectiveness, while theoretical analysis and controlled experiments support the complementary roles of the two paths. These findings highlight the value of preserving historical evidence while adapting its use, offering a practical design for long-term personalized parametric memory.

\section*{Limitations}
Although DPPM advances cross-session personalization through dual-path memory, our study has two limitations. First, computational constraints limit DPPM experiments to 2B--7B backbones, leaving larger-scale validation for future work. Second, our memory interface currently processes textual interaction histories. Extending it to images and speech is a natural next step.

\bibliography{custom}

\appendix
\section{AI Use Statement}
\label{app:ai-use}
Generative AI tools assisted with method implementation, proof drafting, manuscript writing and revision, and figure creation. They also supported discussions and refinement of the interpretation of experimental results. The authors reviewed the implementation and proofs and tested the code. The authors take full responsibility for the final content, including all AI-assisted text, claims, and artifacts.

\section{Experimental Configuration}
\label{app:configuration}
For each dataset and backbone, we train a separate memory module using the splits in Section~\ref{sec:experimental-setup}. The released D2L tokenizer divides interaction histories into chronological segments. The compiler, decoder, shared adapter, and backbone remain frozen. We evaluate final-epoch checkpoints using the settings in Table~\ref{tab:configuration}.

\begin{table}[htbp]
  \centering
  \small
  \setlength{\tabcolsep}{3pt}
  \begin{tabular*}{\columnwidth}{@{\extracolsep{\fill}}lr@{}}
    \toprule
    \textbf{Setting} & \textbf{Value} \\
    \midrule
    Optimizer & AdamW \\
    Learning rate & $10^{-3}$ \\
    Weight decay & $0.01$ \\
    Training epochs & 5 \\
    Questions per update & 1 \\
    Gradient norm clipping & 1.0 \\
    Training seeds & 42--46 \\
    Latent width $d$ / key width $K$ & 512 / 4 \\
    RMS floor $\epsilon$ / readout floor $\delta$ & $10^{-6}$ / $10^{-4}$ \\
    \bottomrule
  \end{tabular*}
  \caption{DPPM cross-entropy training settings.}
  \label{tab:configuration}
\end{table}

For the auxiliary objectives, answer CE remains active. OPD samples 16 or 64 student tokens and uses teacher-scored reverse KL, with a one-epoch warmup, auxiliary weight 0.5, and one auxiliary update every four examples. GRPO samples eight answer actions with correctness rewards. Latency is measured on one A800 at batch size 1, with one warmup and three timed forwards per question.

\section{Significance Analysis}
\label{app:significance}
We compare DPPM and RPMem on Qwen3-4B-Instruct-2507 across the seven component metrics in Table~\ref{tab:main-results}, excluding the derived Avg. column. For each metric, we pair runs with the same training seed (42--46) and apply a two-sided paired $t$-test to their accuracy differences. Holm correction controls the family-wise error rate at 0.05 within this seven-test family. The analysis assumes independent seeds and approximately normal seed-level differences on fixed evaluation sets.

\begin{table}[htbp]
  \centering
  \small
  \setlength{\tabcolsep}{2pt}
  \begin{tabular*}{\columnwidth}{@{\extracolsep{\fill}}lrrr@{}}
\toprule
\textbf{Metric} & \textbf{Gain (pp)} & \textbf{Raw $p$} & \textbf{Adj. $p$}  \\
\midrule
\multicolumn{4}{l}{\textit{PersonaMem-v2}} \\
Overall & +8.50 & $1.87\!\times\!10^{-6}$ & $1.31\!\times\!10^{-5}$ \\
Self & +5.23 & $1.04\!\times\!10^{-4}$ & $6.25\!\times\!10^{-4}$ \\
Current & +5.00 & $2.48\!\times\!10^{-4}$ & $1.24\!\times\!10^{-3}$ \\
\midrule
\multicolumn{4}{l}{\textit{PrefEval}} \\
Overall & +4.32 & $1.50\!\times\!10^{-3}$ & $6.02\!\times\!10^{-3}$ \\
10 turns & +4.96 & $2.19\!\times\!10^{-3}$ & $6.56\!\times\!10^{-3}$ \\
70 turns & +4.33 & $7.65\!\times\!10^{-3}$ & $9.47\!\times\!10^{-3}$ \\
300 turns & +3.67 & $4.73\!\times\!10^{-3}$ & $9.47\!\times\!10^{-3}$ \\
\bottomrule
\end{tabular*}

  \caption{DPPM versus RPMem on Qwen3-4B-Instruct-2507. Gains are in percentage points. Raw $p$-values use five paired training runs, and adjusted values use Holm correction across all seven metrics.}
  \label{tab:significance}
\end{table}

All seven mean differences favor DPPM, and every Holm-adjusted $p$-value is below 0.01. The gains remain significant for both Overall metrics and all five diagnostic or interval-specific measures in this comparison.

\section{Implementation Details}
\label{app:implementation}
\paragraph{Representations and initialization.}
Each compiled memory $m_t$ contains vectors indexed by layer, module, and latent slot. The method equations describe one such vector $q_t$ for $T\geq1$ segments. Learned projections are shared across positions, while memory states are maintained separately. We use non-affine LayerNorm for $z_t$, with numerical constant $10^{-5}$, and compute consolidation in float32. Initially, $W_E$, $W_\beta$, $u$, and $v$ are zero, $b_\beta=-\mathbf{1}_d$, and entries of $W_K$ follow $\mathcal{N}(0,0.04^2)$. Key normalization uses an L2-norm floor of $10^{-12}$.

\paragraph{Scale calibration.}
\label{app:calibration}
For $\xi\in\mathbb{R}^d$, let $\operatorname{RMS}(\xi)=\|\xi\|_2/\sqrt d$. With target scale $\rho\geq0$ and a small positive floor $\epsilon$, the shared rescaling operator is
\begin{equation}
 \mathcal{R}(\xi,\rho)=\frac{\rho\xi}{\max\{\operatorname{RMS}(\xi),\epsilon\}}.
 \label{eq:dppm-calibration}
\end{equation}
When the floor is inactive, rescaling preserves direction and sets the output RMS to $\rho$. The calibrated Evidence representation is $e=\mathcal{R}(\bar e,\rho_E)$, with $\rho_E\geq0$ defined by
\begin{equation}
 \rho_E^2=\frac{1}{d}\sum_{t=1}^T w_t^\top(q_t\odot q_t).
 \label{eq:dppm-evidence-scale}
\end{equation}

The fused representation uses $h=\mathcal{R}(\bar h,\rho_H)$ with nonnegative target scale
\begin{equation}
 \rho_H^2=\frac{1}{d}\sum_{j=1}^d[\alpha_je_j^2+(1-\alpha_j)r_j^2].
 \label{eq:dppm-fusion-scale}
\end{equation}

\paragraph{Delta normalization and initialization.}
\label{app:delta-normalization}
Using the same keys and gates as $S_t$, the normalizer tracks constant-target writes,
\begin{equation}
 N_t=N_{t-1}+k_t[\beta_t\odot(\mathbf{1}_d-N_{t-1}^\top k_t)]^\top.
 \label{eq:dppm-normalizer}
\end{equation}
Both recurrences start at $t=2$, with $S_1=k_1q_1^\top$ and $N_1=k_1\mathbf{1}_d^\top$. In Equation~\ref{eq:dppm-readout}, $[\xi]_\delta=\max(\xi,\delta)$ coordinate-wise, with a small positive floor $\delta$. The Delta representation is $r=\mathcal{R}(\bar r,\rho_D)$, where $\rho_D\geq0$ and
\begin{equation}
 \rho_D^2=\frac{1}{Td}\sum_{t=1}^T\|q_t\|_2^2.
 \label{eq:dppm-delta-scale}
\end{equation}

\paragraph{Evaluation conventions.}
The main DPPM, RPMem, and standalone-path comparisons use seeds 42--46. Auxiliary-objective variants use seed 42, with privileged teacher state confined to training. Their two-dataset Avg. is the arithmetic mean of the displayed Overall scores. PersonaMem-v2 Overall averages all 5,000 questions, while Self and Current select the annotations \texttt{who=self} and \texttt{updated=False}. PrefEval Overall averages all 1,620 instances; each interval averages three preference forms. Category gains and attribution subsets describe answer behavior.

\paragraph{Latent-retention protocol.}
The diagnostic uses frozen seed-42 checkpoints without additional training. Each of 24 profiles has a neutral prefix, one preference segment, and 32 fixed distractors. We move the preference to ages 0, 8, 16, 24, and 32. For each paired preference, the latent sensitivity in the main text is divided by that profile's age-zero sensitivity before averaging profiles. The curves measure latent-response retention with fixed history length and distractors.

\section{Analysis of the Two Memory Paths}
\label{app:memory-proofs}
Fix the learned parameters and let $Q=(q_1,\ldots,q_T)$ contain the segment representations at one latent position. The first two results compare sequences differing only at segment $i$; primes mark perturbed quantities. We use Euclidean and induced operator norms, with $\|\xi\|_\infty=\max_j|\xi_j|$. The results respectively characterize recurrent attenuation, direct access to evidence, and the role of order-sensitive updates.

\subsection{Conditional Recurrent Attenuation}
\begin{proposition}[Contractive state transitions]
\label{prop:recurrent-contraction}
Let $\mu_t=F_t(\mu_{t-1},q_t)$ and $\mu'_t=F_t(\mu'_{t-1},q'_t)$, with $q_t=q'_t$ for $t>i$. For each $t>i$, assume $F_t(\cdot,q_t)$ is continuously differentiable near the line segment joining $\mu_{t-1}$ and $\mu'_{t-1}$, and $\|D_\mu F_t\|_2\leq\kappa_t$ throughout it, where $\kappa_t\geq0$. Then
\begin{equation}
 \|\mu_T-\mu'_T\|_2\leq\Bigl(\prod_{t=i+1}^T\kappa_t\Bigr)\|\mu_i-\mu'_i\|_2.
 \label{eq:proof-contraction}
\end{equation}
If $0\leq\kappa_t\leq\kappa<1$, the multiplier is at most $\kappa^{T-i}$.
\end{proposition}
\begin{proof}
Set $\Delta_t=\mu_t-\mu'_t$ and $J_t(s)=D_\mu F_t(\mu'_{t-1}+s\Delta_{t-1},q_t)$. Integration along the connecting line segment gives
\begin{equation}
 \Delta_t=\int_0^1J_t(s)\Delta_{t-1}\,ds.
\end{equation}
Thus $\|\Delta_t\|_2\leq\int_0^1\kappa_t\|\Delta_{t-1}\|_2\,ds=\kappa_t\|\Delta_{t-1}\|_2$. Iterating this inequality proves the bound.
\end{proof}
To connect this condition to gated consolidation, consider a state-dependent gate $g_t(\mu)\in(0,1)^d$ and let $F_t(\mu)=g_t(\mu)\odot\mu+(\mathbf{1}_d-g_t(\mu))\odot q_t$. Fixing $q_t$ gives
\begin{equation}
 D_\mu F_t=\operatorname{diag}(g_t)+\operatorname{diag}(\mu-q_t)D_\mu g_t.
\end{equation}
Using $\|\operatorname{diag}(\xi)\|_2=\|\xi\|_\infty$, uniform bounds $\|g_t\|_\infty\leq\gamma_t$, $\|\mu-q_t\|_\infty\leq R_t$, and $\|D_\mu g_t\|_2\leq L_t$ permit $\kappa_t=\gamma_t+R_tL_t$. This bound accounts for both state retention and gate sensitivity; values below one ensure contraction.

\subsection{Direct Evidence Sensitivity}
Direct pooling in Equations~\ref{eq:dppm-evidence}--\ref{eq:dppm-pooling} yields a position-independent bound on the raw pooled vector, before RMS calibration.
\begin{proposition}[Bounded-score pooling]
\label{prop:direct-evidence}
Suppose $\max_t s_{t,j}-\min_t s_{t,j}\leq B$ for each $j$, where $B\geq0$ is independent of $T$. Then $w_{i,j}\geq b_T=[1+(T-1)e^B]^{-1}$. If only $q_i$ changes to $q'_i=q_i+\eta$, with $\eta\in\mathbb{R}^d$ and unchanged score vector $s'_i=s_i$, then
\begin{equation}
 \|\bar e'-\bar e\|_2\geq b_T\|\eta\|_2.
 \label{eq:proof-evidence}
\end{equation}
\end{proposition}
\begin{proof}
Since $s_{t,j}-s_{i,j}\leq B$, dividing the softmax denominator by $e^{s_{i,j}}$ gives
\begin{equation}
 w_{i,j}=\frac{1}{1+\sum_{t\ne i}e^{s_{t,j}-s_{i,j}}}\geq b_T.
\end{equation}
Since all scores are unchanged, so are all weights. Hence $\bar e'-\bar e=\operatorname{diag}(w_i)\eta$, and $\|\bar e'-\bar e\|_2^2=\sum_jw_{i,j}^2\eta_j^2\geq b_T^2\|\eta\|_2^2$. Taking square roots proves the claim.
\end{proof}
Normalization bounds the scores independently of history length. Non-affine LayerNorm gives $\|\operatorname{LN}(q)\|_2\leq\sqrt d$. Cauchy--Schwarz gives $B=2\sqrt d\max_j\|(W_E)_{j,:}\|_2$. Also, $\operatorname{LN}(q+\zeta\mathbf{1}_d)=\operatorname{LN}(q)$ for scalar $\zeta$, giving score-preserving perturbations in real arithmetic.

For fixed $B$, $b_T=\Theta(T^{-1})$, independently of the source position $i$. If a nonzero perturbation also induces $\|\Delta\mu_i\|_2\leq L_{\mathrm{in}}\|\eta\|_2$, the raw Evidence lower bound exceeds the recurrent upper bound whenever $b_T>L_{\mathrm{in}}\kappa^{T-i}$. Here $L_{\mathrm{in}}$ bounds input sensitivity and $\kappa$ bounds contraction.

\subsection{The Need for Ordered Revision}
Direct evidence access preserves a contribution from each segment, but content-only pooling discards their order. For a permutation $\pi$, write $\pi Q=(q_{\pi(1)},\ldots,q_{\pi(T)})$.
\begin{proposition}[Order cannot be recovered from pooling]
\label{prop:pooling-order}
Under the Evidence scoring and calibration rules, $e(\pi Q)=e(Q)$ for every permutation $\pi$ of fixed segment representations. If $Q$ and $\pi Q$ require different answers to the same query, a deterministic predictor using only that query and $e(Q)$ cannot be correct on both.
\end{proposition}
\begin{proof}
Permuting inputs reorders scores and weights together. Both $\bar e=\sum_t w_t\odot q_t$ and $\rho_E^2=d^{-1}\sum_t w_t^\top(q_t\odot q_t)$ remain unchanged, and therefore so does $e=\mathcal{R}(\bar e,\rho_E)$. Identical predictor inputs produce the same answer, which cannot equal two distinct targets.
\end{proof}
Oppositely ordered conflicting updates under a last-update-wins rule illustrate this obstruction. Without separately encoded chronology, this motivates order-sensitive Delta updates.

\paragraph{Design implications.}
Under the stated conditions, direct Evidence accumulation mitigates the repeated attenuation of recurrent propagation. Its lower bound decreases with history length but carries no additional age factor. Pooling's permutation invariance motivates the complementary Delta path for ordered revision. Together, the results support combining direct evidence access with sequential state updates.

\end{document}